\documentclass[runningheads]{llncs}
\usepackage{graphicx}
\usepackage{comment}
\usepackage{latexsym}
\usepackage{amsmath}
\usepackage{mathtools}
\usepackage{amssymb}
\usepackage{comment}
\usepackage{stmaryrd}
\usepackage{algorithm}
\usepackage{algpseudocode}
\usepackage{multirow}
\begin{document}
\newcommand{\PP}[1]{{\color{red}{#1}}}

\newcommand{\alg}[1]{\mathsf{#1}}
\newcommand{\Prover}{\alg{P}}
\newcommand{\Verifier}{\alg{V}}
\newcommand{\Simulator}{\alg{S}}
\newcommand{\PPT}{\alg{PPT}}
\newcommand{\isom}{\cong}
\newcommand{\from}{\stackrel{\scriptstyle R}{\leftarrow}}
\newcommand{\handout}[5]{
	\noindent
	\begin{center}
		\framebox{
			\vbox{
				\hbox to 5.78in { {\bf Hybrid Systems} \hfill #2 }
				\vspace{4mm}
				\hbox to 5.78in { {\Large \hfill #5  \hfill} }
				\vspace{2mm}
				\hbox to 5.78in { {\it #3 \hfill #4} }
			}
		}
	\end{center}
	\vspace*{4mm}
}

\newcommand{\ho}[5]{\handout{#1}{#2}{Guide:
		#3}{#4}{#5}}
\newcommand{\al}{\alpha}
\newcommand{\Mod}[1]{\ (\mathrm{mod}\ #1)}
\newcommand{\nd}{\wedge}
\newcommand{\defn}{\coloneqq}
\newcommand{\Z}{\mathbb Z}
\newcommand{\real}{\mathbb{R}}
\newcommand{\nat}{\mathbb{N}}
\newcommand{\calH}{\mathcal{H}}
\newcommand{\reduce}[1]{#1^{\text{red}}}
\newcommand{\loc}{Q}
\newcommand{\calX}{\mathcal{X}}
\newcommand{\poly}[1]{Poly(#1)}
\newcommand{\inv}{\textit{I}}
\newcommand{\flow}{\textit{F}}
\newcommand{\guard}{Guard}
\newcommand{\eguard}{\mathcal{G}}
\newcommand{\edges}{\textit{E}}
\newcommand{\dist}[1]{Dist(#1)}
\newcommand{\itstar}{\item[$\bigstar$]}
\newcommand{\vertx}{V}
\newcommand{\gredge}{E}
\newcommand{\ewt}{W}
\newcommand{\ec}{E_c}
\newcommand{\ep}{E_p}
\newcommand{\graph}{G}
\newcommand{\initnode}{I_0}
\newcommand{\initst}{s_{\textit{init}}}
\newcommand{\stablest}{s_0}
\newcommand{\pf}{P} 
\newcommand{\Path}{\sigma}
\newcommand{\pathend}[1]{{#1}_{\textit{end}}}
\newcommand{\infpath}{\sigma_{\infty}}
\newcommand{\len}[1]{len(#1)}
\newcommand{\indx}{I}
\newcommand{\union}{\cup}
\newcommand{\bigunion}{\bigcup}
\newcommand{\intersect}{\cap}
\newcommand{\bigintersect}{\bigcap}
\newcommand{\ball}[2]{B_{#1}(#2)}
\newcommand{\prob}{Pr}
\newcommand{\probpath}[2]{P_{#1}(#2)}
\newcommand{\tildeprob}{\tilde{Pr}}
\newcommand{\fpath}[1]{\textit{Paths}_{\textit{fin}}(#1)}
\newcommand{\cyl}[1]{\textit{Cyl}(#1)}
\newcommand{\apath}[1]{Paths(#1)}
\newcommand{\ipath}[1]{\textit{Paths}(#1)}
\newcommand{\spath}{\mathcal{SP}}
\newcommand{\scyl}{\mathcal{SC}}
\newcommand{\edgeset}{\mathcal{E}}
\newcommand{\sigal}{\mathcal{F}}
\newcommand{\states}{S}
\newcommand{\ctran}{\rightarrow_c}
\newcommand{\ptran}{\rightarrow_p}
\newcommand{\tran}{P}
\newcommand{\stran}{S_\rightarrow}
\newcommand{\sstran}[1]{{#1}_{\rightarrow}}
\newcommand{\metric}{d}
\newcommand{\weight}{W}
\newcommand{\restrict}[2]{#1|#2}
\newcommand{\wdtmc}{\mathcal{M}_W}
\newcommand{\boundary}[1]{\partial(#1)}
\newcommand{\norm}[2]{\lvert\!\lvert #1 \rvert\!\rvert}
\newcommand{\pathset}[3]{\Sigma_{#1}^{#2}#3}
\newcommand{\partition}{\mathcal{P}}
\newcommand{\size}[1]{\left\vert #1\right\vert}
\newcommand{\abtrct}[2]{#1/#2}
\newcommand{\face}{\mathbb{F}}
\newcommand{\cycle}{\mathcal{C}}
\newcommand{\pwavg}[1]{S_{#1}}
\newcommand{\ldecomp}[1]{#1^{\mathcal{L}}}
\newcommand{\fdecomp}[1]{#1^d}
\newcommand{\pcdecomp}[1]{#1^{\mathcal{SP}\union\mathcal{SC}}}
\newcommand{\tildewt}{\tilde{W}}
\newcommand{\rhost}{\rho^*}
\newcommand{\pst}[1]{p^{st}(#1)}
\newcommand{\mstep}[1]{\overset{\mathrm{#1}}{\leadsto}}
\newcommand{\cov}{Cov}
\newcommand{\var}{Var}
\newcommand{\linineq}[1]{L_{#1}}
\newcommand{\coreq}[1]{{#1}_{eq}}
\newcommand{\reach}{\rightsquigarrow}

\newcommand{\ubar}[1]{\text{\b{$#1$}}}
\newcommand{\ap}{AP}
\newcommand{\until}[1]{U^{\leq {#1}}}
\newcommand{\infuntil}{U}
\newcommand{\true}{\top}
\newcommand{\false}{\bot}
\newcommand{\opAnd}{\wedge}
\newcommand{\Or}{\vee}
\newcommand{\some}[1]{\Diamond^{\leq {#1}}}
\newcommand{\infsome}{\Diamond}
\newcommand{\all}[1]{\Box^{\leq {#1}}}
\newcommand{\infall}{\Box}
\newcommand{\ineqs}{\bowtie}
\newcommand{\next}{\bigcirc}
\newcommand{\Label}{L}
\newcommand{\assign}{V}
\newcommand{\pathvars}{\Pi}
\newcommand{\finpathvars}{\Pi_{fin}}
\newcommand{\semantics}[1]{{\llbracket #1\rrbracket}}
\newcommand{\shift}[2]{#1^{(#2)}}
\newcommand{\notmodels}{\nvDash}
\newcommand{\thresh}[1]{threshold_{#1}}
\newcommand{\nullhyp}{H_0}
\newcommand{\althyp}{H_1}
\newcommand{\sample}[1]{\mathbf{#1}}
\newcommand{\statistic}{T}
\newcommand{\falsepos}{\alpha_{FP}}
\newcommand{\falseneg}{\alpha_{FN}}
\newcommand{\range}[1]{[#1]}
\newcommand{\complmnt}[1]{{#1}^c}
\newcommand{\pred}{\mathbb{P}}
\newcommand{\pdf}{f}
\newcommand{\cdf}{F}
\newcommand{\bayes}{\mathbb{B}}
\newcommand{\samsp}{\Omega}
\newcommand{\extset}[2]{{#1}^{#2 +}}
\newcommand{\deductset}[2]{{#1}^{#2 -}}

\newcommand{\acts}{A}
\newcommand{\trans}{T}
\newcommand{\mdp}{\text{MDP}}
\newcommand{\rmdp}{\text{WMDP}}
\newcommand{\M}{\mathcal{M}}
\newcommand{\Mw}{\mathcal{M}_W}
\newcommand{\policy}{\rho}
\newcommand{\optpol}{\rho^*}
\newcommand{\polset}{\Gamma}
\newcommand{\dtmc}{\text{DTMC}}
\newcommand{\rdtmc}{\text{WDTMC}}
\newcommand{\scc}{\text{SCC}}
\newcommand{\bscc}{\text{BSCC}}
\newcommand{\sccset}{\mathcal{S}_G}
\newcommand{\Md}{\mathcal{M}_D}
\newcommand{\init}{\textit{Init}}
\newcommand{\borel}{\mathcal{B}}
\newcommand{\dtshs}{\text{dt-SHS}}
\newcommand{\shs}{\text{SHS}}
\newcommand{\currst}{s_{curr}}
\newcommand{\reward}{R}
\newcommand{\avg}{E}
\newcommand{\meanpay}{\gamma}
\newcommand{\support}[1]{\textit{support}(#1)}
\newcommand{\pphs}{\text{PPHS}}
\newcommand{\convp}[1]{\overset{#1}{\rightarrow}}
\newcommand{\submin}[1]{#1_{\textit{min}}}
\newcommand{\submax}[1]{#1_{\textit{max}}}
\newcommand{\bsccset}{\mathcal{B}}
\newcommand{\B}{B}
\newcommand{\Sc}{S_c}
\newcommand{\calD}{\mathcal{D}}
\newcommand{\calDw}{\mathcal{D}_W}
\newcommand{\eq}[1]{{#1}_{\textit{eq}}}
\newcommand{\memless}[1]{{#1}_{\textit{mem}}}

\title{Abstraction-based Probabilistic Stability Analysis of Polyhedral Probabilistic Hybrid Systems
}
\titlerunning{Stability of PPHS}
%

\author{Spandan Das\inst{1}\orcidID{0000-0002-1995-2592} \and
Pavithra Prabhakar\inst{1} 
}
\authorrunning{S. Das and P. Prabhakar}
%
\institute{Kansas State University, Manhattan, Kansas, USA\\
\email{\{spandan,pprabhakar\}@ksu.edu}
}
%
\maketitle              
%
\begin{abstract}
In this paper, we consider the problem of probabilistic stability analysis of a subclass of Stochastic Hybrid Systems, namely, \emph{Polyhedral Probabilistic Hybrid Systems ($\pphs$)}, where the flow dynamics is given by a polyhedral inclusion, the discrete switching between modes happens probabilistically at the boundaries of their invariant regions and the continuous state is not reset during switching. 
We present an abstraction-based analysis framework that consists of constructing a finite Markov Decision Processes ($\mdp$) such that verification of certain property on the finite $\mdp$ ensures the satisfaction of probabilistic stability on the $\pphs$.
Further, we present a polynomial-time algorithm for verifying the corresponding property on the $\mdp$.
Our experimental analysis demonstrates the feasibility of the approach in successfully verifying probabilistic stability on $\pphs$ of various dimensions and sizes.

\end{abstract}
%
%
%
\section{Introduction}
Stability is a fundamental property of hybrid control systems that stipulates that, small changes in initial state or inputs lead to only small deviations in the behaviors of the system, and that the effect of those perturbations on the system behaviors diminishes over time.
Probabilistic stability~\cite{rutten2004mathematical} extends this notion to stochastic systems which model uncertainties in the environment. 
In this paper, we study probabilistic stability of a certain kind of Stochastic Hybrid Systems ($\shs$), namely, Polyhedral Probabilistic Hybrid Systems ($\pphs$), which are $\shs$ where flow rates are constrained by linear inequalities, and the mode switches probabilistically when the continuous state satisfies certain linear constraints. 
These systems are powerful due to the non-determinism in the dynamics, and can precisely over-approximate linear hybrid systems~\cite{prajna2003analysis,prabhakar2016hybridization} through a process called hybridization.

While safety analysis of $\pphs$ \cite{prajna2004safety,lal2018hierarchical,clarke2003abstraction,alur2003counter,lal2019counterexample} and stability of polyhedral hybrid systems in the non-probabilistic setting \cite{prabhakar2013abstraction,prabhakar2016hybridization,prabhakar2013decidability} have been extensively investigated, stability analysis remains an open problem. 
Classically, stability analysis techniques have been built on the notion of Lyapunov functions\cite{Branicky98,davrazos2001review,liberzon2003switching,van2000introduction} and have been extended to the setting of stochastic systems\cite{kozin1969survey,verdejo2012stability,zhang2014exponential,prajna2003analysis}. 
A detailed study on sufficient conditions for stability of $\shs$ based on Lyapunov functions has been performed in~\cite{teel2014stability}. Almost sure exponential stability~\cite{cheng2012almost,cheng2018almost,do2020almost,hu2008almost} and asymptotic stability in 
distribution~\cite{yuan2003asymptotic,wang2019asymptotic} using Lyapunov functions have also been investigated for $\shs$.
While Lyapunov functions provide a certificate of stability, computing them is quite challenging as it involves exploring complex polynomial templates and deducing coefficients of such templates by solving non-linear optimization problems\cite{Branicky98,giesl2015review}. 
An alternative but much less explored method involves an abstraction-based analysis, that has shown promise in the non-probabilistic setting~\cite{prabhakar2013abstraction}, and more recently in the probabilistic setting in low dimensions~\cite{das2022stability}. 
In this paper, we present an abstraction-based analysis technique for probabilistic stability analysis of $\pphs$ by abstracting the system to a finite Markov Decision Process and checking that an infinite path converges to an equilibrium point in expectation in the abstract system.

Broadly, our approach is to abstract a $\pphs$ to a finite Markov Decision Process ($\mdp$) with edge weights, and calculate the expected mean payoff of an infinite path under the worst possible policy~\cite{chatterjee2012games,singh1994reinforcement,gimbert2007pure,kvretinsky2017efficient}. 
We show that, if mean payoff of an infinite path of the abstract $\mdp$ is negative, then the $\pphs$ (which has an infinite $\mdp$ semantics) is stable. 
While finding optimal policies for maximum expected mean payoff is computationally expensive, we present a polynomial-time algorithm to compute this maximum expected mean payoff, which suffices for our purpose to deduce stability.
This requires decomposing the $\mdp$ into communicating $\mdp$s, calculating worst case expected mean payoff of a path in each of these $\mdp$s and, combining these weights in a suitable manner to obtain the worst case expected mean payoff of a path of the original $\mdp$~\cite{puterman2014markov,kvretinsky2017efficient,chatterjee2013faster}.

The main contributions of this paper are:
\begin{itemize}
    \item An approach for abstraction of $\pphs$ (which has an infinite $\mdp$ semantics) to a finite $\mdp$ with edge weights, such that probabilistic stability of $\pphs$ can be inferred by checking that the worst case expected mean payoff of a path in the $\mdp$ is negative. 
    \item A polynomial time algorithm to compute maximum expected mean payoff of an infinite path of a finite $\mdp$.
    \item Experimental evaluation on $\pphs$ with varying dimensions and sizes.
\end{itemize}

\section{Preliminaries}\label{sec:prelim}
In this section, we will discuss basic notations and important concepts related to Discrete-time Markov Chain ($\dtmc$), Weighted Discrete-time Markov Chain ($\rdtmc$), Markov Decision Process ($\mdp$), Weighted Markov Decision Process ($\rmdp$), and policies of $\rmdp$.

\subsection{Basic Notations}\label{sec:base-notation}
We denote the set of all natural numbers (excluding $0$) by $\nat$ and the set of all real numbers by $\real$. 
The set of first $n$ natural numbers are denoted by $[n]$.

A distribution over a set $\states$ is a function $d:\states\rightarrow[0,1]$ such that $\sum_{s \in \states} d(s) = 1$, where we assume that the support of $d$, denoted $\support{d} = \{ s \,|\, d(s) > 0\}$, is countable. $d(A|B)$ denotes the probability of event $A$ given $B$, that is, $\frac{d(A\intersect B)}{d(B)}$ (assuming $d(B)\neq 0$). For a real valued function $f:\states\rightarrow\real$, expectation of $f$ under distribution $d$, i.e., $\sum_{s\in\states}f(s)d(s)$, is denoted as $\avg^d[f]$ ($E[f]$ when $d$ is understood from the context).
$\dist{\states}$ denotes the set of all distributions on the set $\states$.

For a vertor $x=(x[1],\dots,x[n])\in\real^n$, $\norm{x}{\infty}$ denotes the infinite norm of $x$, that is, $\max_{i\in[n]}\size{x[i]}$. 
For $x,y\in\real^n$, distance of the point $x$ from point $y$ is given by $\norm{x-y}{\infty}$ and denoted as $\metric(x,y)$. 

\subsection{Markov Decision Process}
 A Markov Decision Process ($\mdp$) is an abstract model with a set of states $\states$ and a set of actions $\acts$, that selects a distribution from $\dist{\states}$ based on the current state $s\in\states$ and the current action $a\in\acts$ chosen at $s$.
\begin{definition}
[MDP] A Markov Decision Process ($\mdp$) is a tuple $\M=(\states,\acts,\trans)$ such that
\begin{itemize}
    \item $\states$ is a nonempty set of states
    \item $\acts$ is a nonempty set of actions
    \item $\trans:\states\times\acts\rightarrow\dist{\states}$ is a mapping from the set $\states\times\acts$ to the set of all distributions on $\states$.  
\end{itemize}
\end{definition}
At any state $s$, an action is chosen non-deterministically from $\acts$.
We use $\trans(s,a,s')$ to denote the probability of going from state $s$ to state $s'$ when action $a$ is chosen, i.e., $\trans(s,a,s')=\alpha(s')$ where $\alpha=\trans(s,a)$.

A finite path of an $\mdp$ $\M$ is an alternating sequence of states and actions, $\Path = s_0,a_1,s_1,a_2,s_2,\dots,s_n$ such that for each $0<i\leq n$, $a_i\in\acts$ and $\trans(s_{i-1},a_i,s_i)>0$. 
We say $n$ is the size (denoted $\size{\Path}$), $s_n$ is the ending state (denoted $\pathend{\Path}$) and $s_0$ is the starting state (denoted $\Path_0$) of the path $\Path$. 
A state $s_2$ is said to be reachable from $s_1$ (denoted $s_1\reach s_2$) if there is a finite path $\Path$ such that $\Path_0=s_1$ and $\pathend{\Path}=s_2$.
$\Path_i$ denotes the $i^{th}$ state $s_i$ and 
$\Path[i:j]$ ($0\leq i\leq j\leq n$) denotes the subpath $s_i,a_{i+1},s_{i+1},\dots,s_j$ of the path $\Path$. 
We say a path $\Path$ is an edge if $\size{\Path}=1$ and infinite if $\size{\Path}=\infty$. 
An edge $e$ is reachable from a state $s$ if $s\reach e_0$.
The set of all edges, finite paths and infinite paths of an $\mdp$ $\M$ are denoted by $\edgeset_\M$ ($\edgeset$ when $\M$ is understood from the context), $\fpath{\M}$ and $\ipath{\M}$ respectively.

We assume that for an $\mdp$ $\M$, the next action is determined based on the current history, i.e., the finite path that has been observed until the most recent time point.
Given any finite path of an $\mdp$ $\M$, a policy is a function that determines the next action.
\begin{definition}[Policy]\label{def:policy}
 A policy $\policy:\fpath{\M}\rightarrow\acts$ on an $\mdp$ $\M$ is a function from the set of finite paths $\fpath{\M}$ to the set of actions $\acts$. 
\end{definition}
A policy is said to be memoryless if the next action is determined based on the current state only. 
Note that, given a memoryless policy $\policy$ for an $\mdp$ $\M$, the probability of transition from state $s$ to $s'$ is uniquely given by $\trans(s,\policy(s),s')$. 
We say a discrete-time Markov chain ($\dtmc$)~\cite{das2022stability} is an $\mdp$ with an associated memoryless policy such that, probability of transition between any two states is uniquely defined.

The set of all possible policies of an $\mdp$ $\M$ is denoted by $\polset_\M$. We abuse notation and write $\polset$ when $\M$ is understood from the context. Given an $\mdp$ $\M$, a policy $\policy$ and an initial distribution $d\in\dist{\states}$ on the set of states $\states$, we define probability of a finite path $\Path$ inductively as~\cite{gimbert2007pure},
\[
\trans_\policy(\Path) = 
\begin{cases}
d(\Path_0) \text{ if }\size{\Path}=0\\
\trans_\policy(\Path')\cdot\policy(\Path')\trans(s_{\size{\Path}-1},\policy(\Path'),s_{\size{\Path}}) \text{ otherwise},
\end{cases}
\]
where $\Path' = \Path[0:\size{\Path}-1]$. 
For this work, we will assume that the initial distribution of states of an $\mdp$ is an indicator function for a unique state $\initst$ known as the {\it initialization point}, i.e., $d(s)=1$ iff $s=\initst$.

\subsection{Weighted Markov Decision Process}

\begin{figure}[h]
	\centering
	\setlength\abovecaptionskip{-0pt}
	\setlength\belowcaptionskip{-5pt}
	\includegraphics[width=10cm]{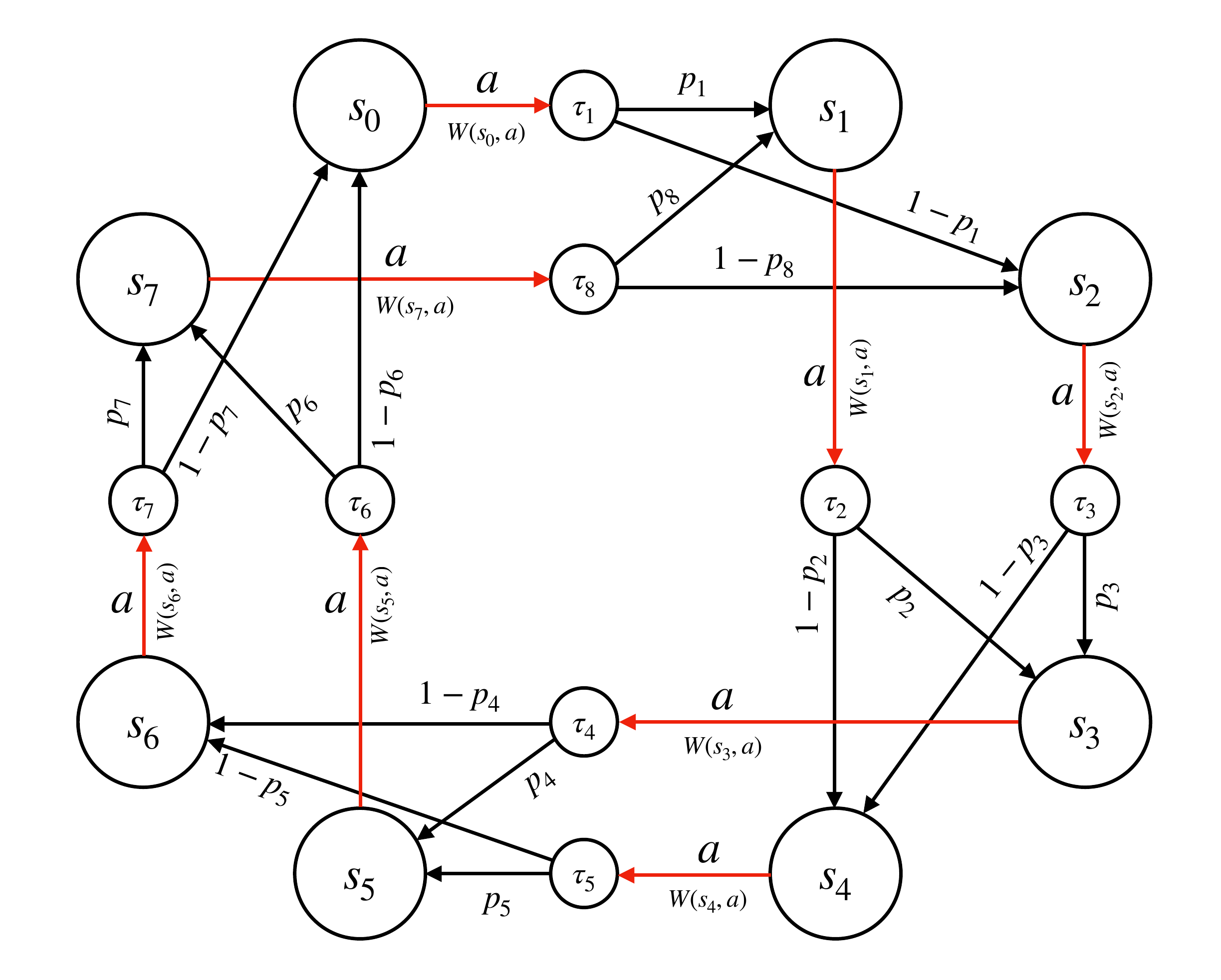}
	\caption{Graphical depiction of a sample MDP}
	\label{fig:mdp}
\end{figure}

We extend $\mdp$ to Weighted $\mdp$ ($\rmdp$) by associating a weight to each possible edge. 
Basically, a $\rmdp$ can be observed as a Rewardful $\mdp$, where we gain weights instead of rewards after each action. 
\begin{definition}
[Weighted $\mdp$] A Weighted $\mdp$ ($\rmdp$) is a tuple $\Mw=(\states,\acts,\trans,\weight)$ where $\M=(\states,\acts,\trans)$ is an $\mdp$ and $\weight:\states\times\acts\rightarrow\real$ is a function associating a real weight to each state-action pair.
\end{definition}

\begin{example}\label{example:sample-mdp}
Let us graphically illustrate a sample $\rmdp$ $\Mw$ where $\states=\{s_0,\dots,s_7\}$, $\acts=\{a\}$ and $\trans(s_{i-1},a)=\tau_{i}\in\dist{\states}$ for all $i\in[8]$ where,
\[
\tau_i(s) = 
\begin{cases}
    p_i \quad\text{if }i\Mod{2}=0 \text{ and }s = s_{(i+1)\Mod{8}}\\
    1-p_i \quad\text{if }i\Mod{2}=0 \text{ and }s = s_{(i+2)\Mod{8}}\\
    p_i \quad\text{if }i\Mod{2}=1 \text{ and }s = s_{i\Mod{8}}\\
    1-p_i \quad\text{if }i\Mod{2}=1 \text{ and }s = s_{(i+1)\Mod{8}}\\
    0 \quad\text{otherwise.}
\end{cases}
\]
We depict $\Mw$ in Figure~\ref{fig:mdp} as a graph where $s_i$, $\tau_j$ are nodes of the graph, there is a non-deterministic edge from $s_i$ to $\tau_j$ (marked red) if $\trans(s_i,a)=\tau_j$ and, there is a probabilistic edge from $\tau_j$ to $s_i$ if $\tau_j(s_i)\neq 0$. Each non-deterministic edge $(s_i,\tau_j)$ is labeled with the action $a$ and the weight of the edge $\weight(s_i,a)$. Each probabilistic edge is labeled with its probability.\qed
\end{example}

Note that, under a memoryless policy $\policy$, a $\rmdp$ not only gives a unique probability, but also gives a unique weight of transition from state $s$ to $s'$ (given by $\weight(s,\policy(s))$).
We say a Weighted $\dtmc$ ($\rdtmc$) is a $\rmdp$ associated with a memoryless policy such that, probability and weight of transition between any two sates are uniquely defined.

The weight of a finite path $\Path=s_0,a_1,s_1,\dots,s_n$ of an $\rmdp$ $\Mw$, denoted $\weight(\Path)$, is the sum of weights of all edges that appear on that path, i.e., 
\[\weight(\Path) = \sum_{i=1}^{n}\weight(s_{i-1},a_i).\]
Similarly, Mean payoff of a finite path $\Path$ is the mean weight of all edges appearing on the finite path.
\begin{definition}
[Mean Payoff] For a $\rmdp$ $\Mw$, the mean payoff $\meanpay$ is a function from $\fpath{\Mw}$ to $\real$ defined as,
\[\meanpay(\Path) = \frac{\weight(\Path)}{\size{\Path}}.\]
\end{definition}

\subsection{Maximum Expected Mean Payoff of WMDP}
Given a $\rmdp$, we are interested in finding the maximum expected mean payoff of an infinite path under any policy. 
Here, we formally state this as a problem and discuss an efficient algorithm for its solution.
\begin{problem}\label{prob:max-expt-payoff}
Given a $\rmdp$ $\Mw$, find the maximum expected mean payoff of an infinite path under any policy, i.e., 
\[
\sup_{\policy\in\polset}\left(\limsup_{n\rightarrow\infty}\avg^{\trans^\policy_n}[\meanpay]\right),
\]
where $\trans^\policy_n$ is a distribution on $\fpath{\Mw}$ given by $\trans^\policy_n(\Path)=\trans_\policy(\Path\mid \size{\Path}=n)$.
\end{problem}
We denote $\limsup_{n\rightarrow\infty}\avg^{\trans^\policy_n}[\meanpay]$ as $\avg_\policy[\meanpay]$. 
Now, let us briefly discuss an algorithm for solving the above problem.

\subsubsection{Algorithm for calculating maximum expected mean payoff:}\label{algo:lp-gain}
Our goal is to solve Problem~\ref{prob:max-expt-payoff} for a $\rmdp$ with finite set of states and actions (finite $\rmdp$).
Algorithm 2 (MEC-SI) from~\cite{kvretinsky2017efficient} solves a similar problem for the class of $\rmdp$s with strictly positive weights and the $\limsup$ in the expected mean payoff replaced with $\liminf$, that is, it computes $\sup_{\policy\in\polset}\left(\liminf_{n\rightarrow\infty}\avg^{\trans^\policy_n}[\meanpay]\right)$.
We want to use this algorithm to solve Problem~\ref{prob:max-expt-payoff}. In order to do that, the input $\rmdp$ must be suitably modified to suit the input criterion of Algorithm MEC-SI and, the desired output should be derivable from the output of the algorithm. 
Note that, $\limsup$ of a real sequence can be found by negating the $\liminf$ of the inverted sequence. 
Thus, if we negate all weights of the input $\rmdp$ and find $\sup_{\policy\in\polset}\left(\liminf_{n\rightarrow\infty}\avg^{\trans^\policy_n}[\meanpay]\right)$, we are actually finding the maximum expected payoff of the original $\rmdp$.
Also, if we shift the outputs of a real valued function by a constant real bias, the expectation gets shifted by the same bias. 
Thus, solving Problem~\ref{prob:max-expt-payoff} for a $\rmdp$ is the same as solving Problem~\ref{prob:max-expt-payoff} for the $\rmdp$ after shifting each weight by a constant, and then removing the constant. 
Hence, for a $\rmdp$ with both positive and negative weights, we can solve Problem~\ref{prob:max-expt-payoff} by first negating all the weights, then adding a constant bias to each weight to make them all positive, and finally applying Algorithm MEC-SI on the $\rmdp$ with the modified weights.

Let us now briefly describe Algorithm MEC-SI. The main steps of the algorithm are:
\begin{enumerate}
    \item Achieve Maximal End Component (MEC) decomposition~\cite{chatterjee2013faster} of the input $\rmdp$.
    \item For each MEC, find $\sup_{\policy\in\polset}\left(\liminf_{n\rightarrow\infty}\avg^{\trans^\policy_n}[\meanpay]\right)$ for the induced $\rmdp$ by strategy iteration.
    \item Construct the MEC-quotient (an $\mdp$) using the values obtained in the previous step~\cite{kvretinsky2017efficient}.
    \item Find the maximum reaching probability to a particular state of the MEC-quotient~\cite{kvretinsky2017efficient}.
\end{enumerate}
Note that, the algorithm works in polynomial time~\cite{kvretinsky2017efficient,chatterjee2013faster} if we can solve steps 2 and 4 in polynomial time. 
Note that, step 2 cannot be done in polynomial time as strategy iteration is not guaranteed to converge in polynomial time. 
Also note that, we don't actually need to synthesize an optimal strategy, rather, we only need the optimal gain $\sup_{\policy\in\polset}\left(\liminf_{n\rightarrow\infty}\avg^{\trans^\policy_n}[\meanpay]\right)$ of the MEC. 
Let us show how we can do this in polynomial time. 
If a $\rmdp$ is finite and strongly connected (communicating), that is, each state is reachable from every other state, then there is a linear program (LP) formulation for the optimal gain problem~\cite{puterman2014markov}. 
Since MECs are strongly connected~\cite{kvretinsky2017efficient}, we can obtain optimal gain of an MEC by solving an LP. 
Since an LP can be solved in polynomial time, step 2 can actually be completed in polynomial time as well. 

If an $\mdp$ is finite, maximum reaching probability to a state can be solved by solving an LP~\cite{puterman2014markov}. 
Since MEC-quotient is a finite $\mdp$~\cite{kvretinsky2017efficient}, we can use this LP formulation for step 4. 
Thus, step 4 can also be completed in polynomial time.
So, we have a polynomial-time algorithm for computing maximum expected mean payoff of a $\rmdp$.

\begin{example}
Let us apply our algorithm on the sample $\rmdp$ described in Example~\ref{example:sample-mdp}. First, we negate each weight, i.e., the modified weight of a state-action pair $(s_i,a)$ becomes $-\weight(s_i,a)$. Now, if $\min_{i\in[8]}(-\weight(s_{i-1},a))\leq0$, then we set the constant $c\gets\size{\min_{i\in[8]}(-\weight(s_{i-1},a))}+1$ and $c\gets 0$ otherwise. We add $c$ to each of the modified weights. Thus, the final weight of a state-action pair $(s_i,a)$ becomes $-\weight(s_i,a)+c$, which is strictly positive. We now apply Algorithm MEC-SI on the $\rmdp$ with the modified weights. Note that, the $\rmdp$ is strongly connected, i.e., it has only one MEC (see~\cite{kvretinsky2017efficient}). Thus, we can skip steps 3 and 4 altogether. The maximum expected mean payoff of the sample $\rmdp$ is simply $-(v-c)$, where $v$ is the value obtained from step 2 by solving the linear program for the $\rmdp$ with modified weights.
\end{example}

\section{Polyhedral Probabilistic Hybrid Systems}

\begin{figure}[ht]
	\centering
	\setlength\abovecaptionskip{-0pt}
	\setlength\belowcaptionskip{-5pt}
	\includegraphics[width=10cm]{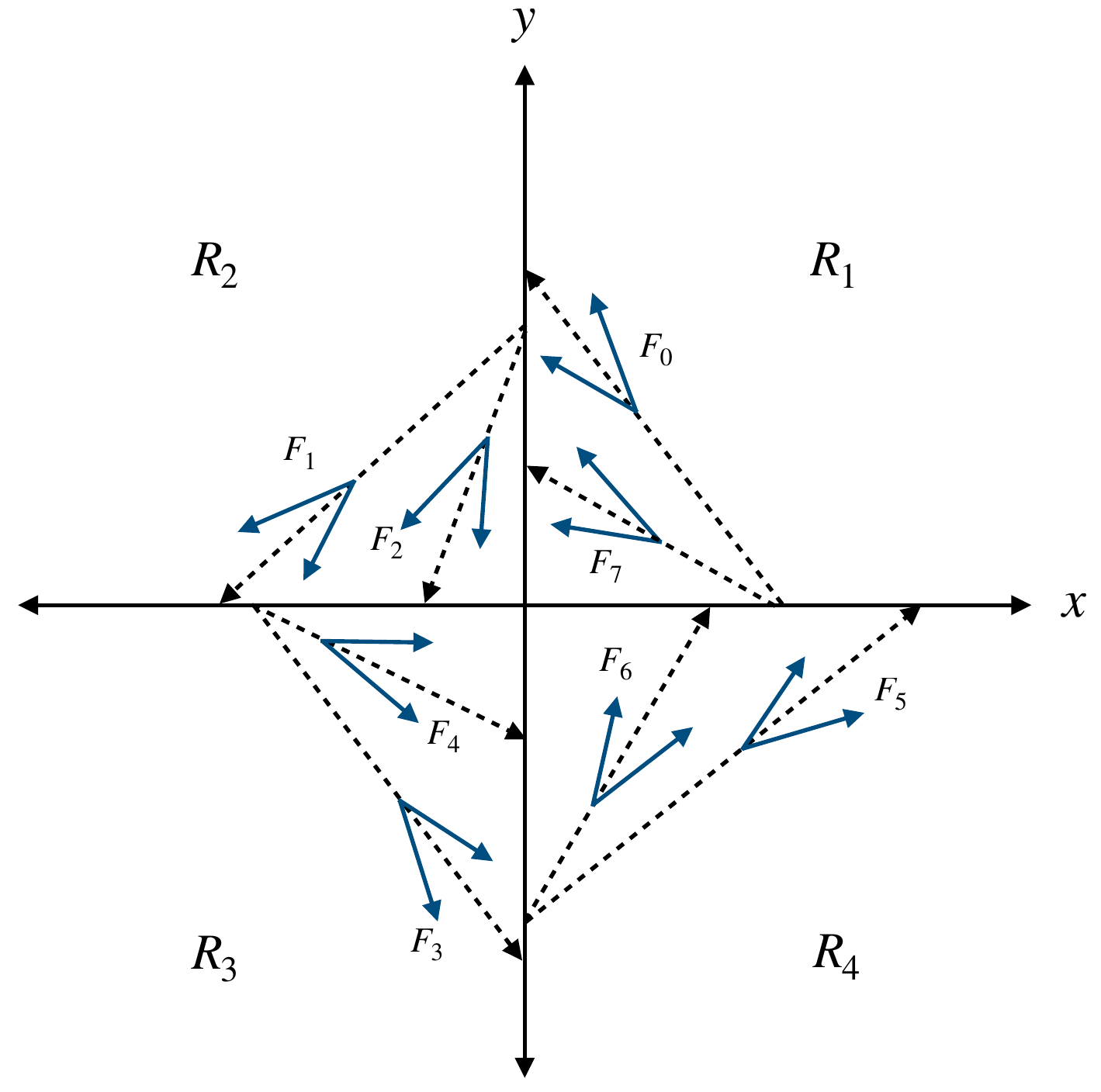}
	\caption{A sample PPHS}
	\label{fig:phs}
\end{figure}

\begin{figure}[ht]
	\centering
	\setlength\abovecaptionskip{-0pt}
	\setlength\belowcaptionskip{-5pt}
    \includegraphics[width=12cm]{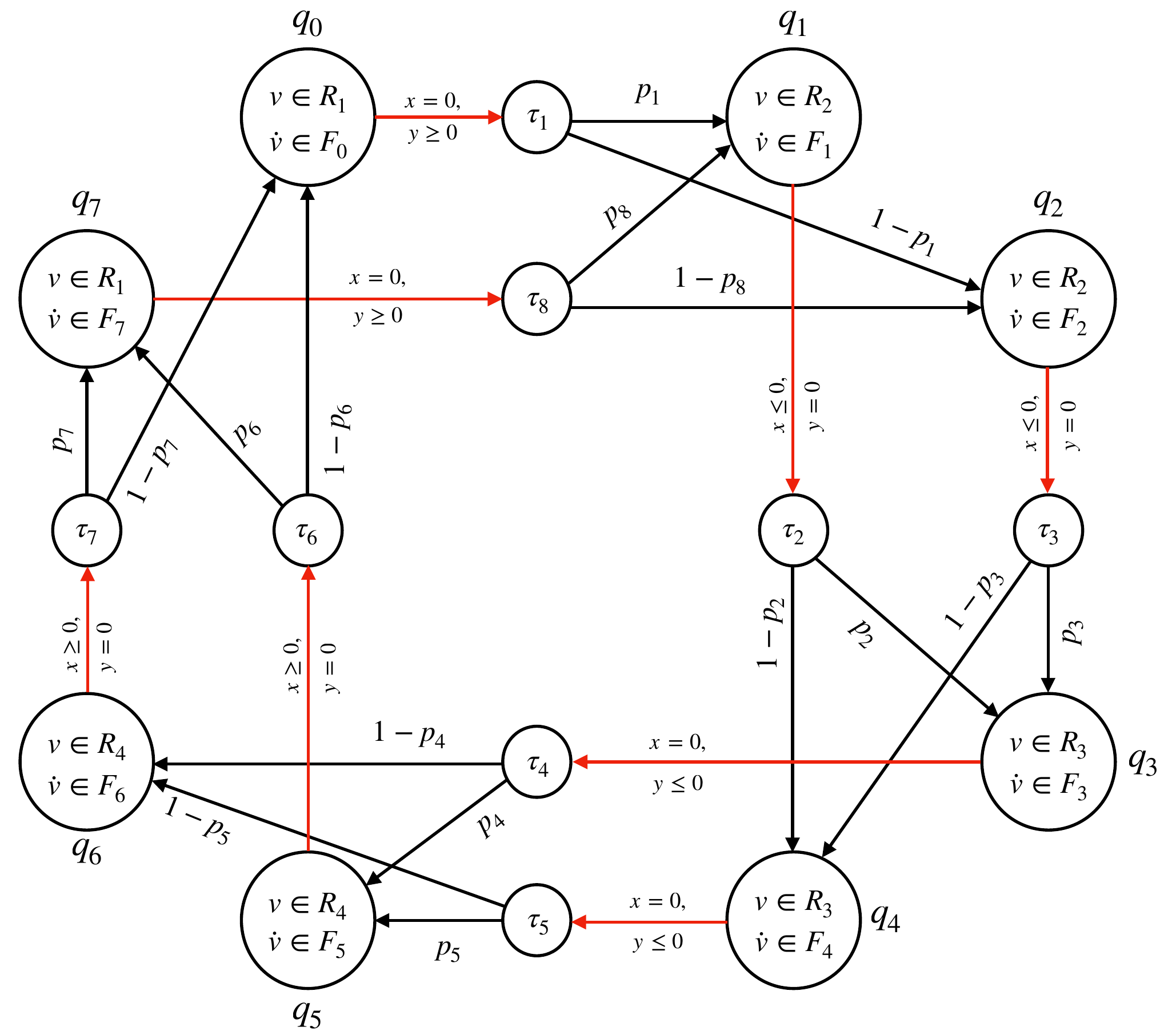}
	\caption{Graphical representation of sample PPHS}
	\label{fig:phs-graph}
\end{figure}

In this section, we define Polyhedral Probabilistic Hybrid System ($\pphs$) and associate a notion of stability to $\pphs$.

\begin{definition}
	[$\pphs$]\label{def-pphs}
	The Polyhedral Probabilistic Hybrid System ($\pphs$) is defined as the tuple 
	$\calH\defn(\loc,\calX,\inv,\flow,\edges)$ where,
	\begin{itemize}
		\item $\loc$ is the set of discrete locations,
		\item $\calX = {\real}^n$ is the continuous state space for
		some $n\in\nat$,
		\item $\inv: \loc\rightarrow\poly{n}$ is the invariant
                  function which assigns a positive scaling invariant polyhedral subset of the state space to each location $q\in\loc$,
		\item $\flow: \loc\rightarrow \poly{n}$ is the Flow function which 
		associates a flow polyhedron to each location $q\in \loc$
		\item $\edges \subseteq
                  \loc \times(\union_{q\in\loc}\face(\inv(q)))\times\dist{Q}$
                  is the probabilistic edge relation such that $(q, f,
                  \zeta)\in\edges$ where for every $(q, f)$, there is a
                  at most one $\zeta$ such that $(q, f, \zeta)  \in
                  \edges$ and $f\in\face(\inv(q))$.
                  $f$ is called a \emph{Guard} of the location $q$.
	\end{itemize}
\end{definition}

Let us describe the semantics of the $\pphs$.
An execution starts from $(q_0,x_0)$ where $q_0\in\loc$ and $x_0\in\inv(q_0)$. It evolves continuously for
some time $T$ according to a flow rate that is chosen non-deterministically from $\flow(q_0)$, until it reaches a facet $f_0$ of $\inv(q_0)$.
Then a probabilistic discrete transition is taken if there is an edge
$(q_0, f_0, \zeta_0)$ and the state $q_0$ is probabilistically changed
to $q_1$ with probability $\zeta_0(q_1)$.
The execution (tree) continues with alternating continuous and
discrete transitions.

Formally, for $x_1, x_2\in\calX$ and $q \in
\loc$, we say that there is a \emph{continuous transition} from $x_1$
to $x_2$ with respect to  $q$ if $x_1,x_2\in\inv(q)$, there exists
$R\geq 0$ and $f\in\flow(q)$ such that $x_2=x_1+f\cdot R$, $x_1 + \flow(q) \cdot t
\in \inv(q)$ for all $0 \leq t < R$ and $x_2 \in
\boundary{\inv(q)}$.
If for all $t\geq 0$, $x_1+\flow(q)\cdot t\in\inv(q)$ then we
say $x_1$ has an infinite edge with respect to $q$.
For two locations $q_1, q_2\in\loc$, we say there is a \emph{discrete
transition} from $q_1$ to $q_2$ with probability $p$ via $\zeta\in\dist{\loc}$ and
$f\in\face(q_1)$ if  $f\subseteq\inv(q_2)$, $(q_1,f,\zeta)\in\edges$
and $p = \zeta(q_2)$.

\begin{example}\label{example:phs}
Let us illustrate $\pphs$ using an example. 
Let $\{q_{i-1}\mid i\in[8]\}$ be the set of discrete locations and $\calX=\real^2$ be the continuous state space for a $\pphs$ $\calH$. 
The four quadrants $R_1$, $R_2$, $R_3$ and $R_4$ are the invariant regions with $R_1$ associated to $q_0$ and $q_7$, $R_2$ associated to $q_1$ and $q_2$, $R_3$ associated to $q_3$ and $q_4$ and, $R_4$ associated to $q_5$ and $q_6$. 
The rate of change of continuous state (flow rate) at location $q_i$ is non-deterministically chosen from the polyhedron $F_i\subseteq\real^2$. 
For each location, either positive $X$, or positive $Y$, or negative $X$, or negative $Y$ axis serves as the guard, that is, the system probabilistically jumps to a new location once the continuous state reaches the guard. 
For example, if the continuous state of the system becomes $(0,y)$ for some $y\in\real$ when the system is in $q_0$, then the system will change its location to either $q_1$ or $q_2$ probabilistically, since positive $Y$ axis is the guard of $q_0$. 
We illustrate the possible evolutions of the system in Figure~\ref{fig:phs}.

We also provide a graphical depiction of the $\pphs$ in Figure~\ref{fig:phs-graph}. 
Note that, each location $q_i$ is labeled with $v\in R_j,\dot{v}\in F_i$, which implies that the continuous state $v$ belongs to the the invariant $R_j$ and rate of change of the continuous state $\dot{v}$ belongs to the flow polyhedron $F_i$ associated to the location.
Each non-deterministic edge outgoing from a location $q_i$ (marked red) is labeled with a guard set and leads to a probability distribution $\tau_{i+1}$ on $\loc$. Probabilistic edges are directed from $\tau_i$ to the next possible locations and labeled with the probability of the target location under $\tau_i$.\qed
\end{example}

We capture the semantics of a $\pphs$ using a $\rmdp$, where continuous transitions are analogous to non-deterministic actions and discrete transitions are equivalent to probabilistic change of state. To reason
about convergence, we need to capture the relative distance of
the states from the equilibrium point, which is captured using edge weights.
Let us fix $0$ as the equilibrium point for the rest of the
section.
The weight on a transition from $(q_1, x_1)$ to $(q_2, x_2)$ captures
the logarithm of the relative distance of $x_1$ and $x_2$ from $0$,
that is, $\log\left(\norm{x_2}{\infty}/\norm{x_1}{\infty}\right)$,
where $\norm{x}{\infty}$ captures the distance of state $x$ from $0$.

\begin{definition}
	[Semantics of $\pphs$]
	Given a $\pphs$ $\calH$, we can construct the $\rmdp$
        $\semantics{\calH}\defn(\states_\calH,\acts_\calH,\trans_\calH,\weight_\calH)$
        where, 
	\begin{itemize}
		\item $\states_\calH=\loc\times\calX$
		\item $\acts_\calH\subseteq \states_\calH\times\states_\calH$ with $((q_1,x_1),(q_2,x_2))\in\acts$ iff $q_1=q_2$ and there is a continuous transition from $x_1$ to $x_2$ with respect to $q_1$.
		\item $\trans_\calH((q_2,x_2)| (q_1,x_1),a)=\zeta(q_2)=p$ if $a=((q_1,x_1),(q_1,x_2))$, there is a discrete transition from $q_1$ to $q_2$ with probability $p$ via $\zeta$ and $f$, and $x_2\in f$. In all other cases, $\trans_\calH((q_2,x_2)| (q_1,x_1),a)=0$.
		\item $\weight_\calH((q_1,x_1),a)=\log\left(\norm{x_2}{\infty}/\norm{x_1}{\infty}\right)$ where $a=((q_1,x_1),(q_1,x_2))$.
	\end{itemize}
\end{definition}

An infinite path $\Path=(q_0,x_0),((q_0,x_0),(q_0,x_1)),(q_1,x_1),\dots$ of the semantics $\rmdp$ $\semantics{\calH}$ is said to converge to $0$ if $\lim_{n\rightarrow\infty}\norm{x_n}{\infty}=0$. 
Thus, $\Path$ converges if and only if $\weight_\calH(\Path)=-\infty$, if and only if $\limsup_{n\rightarrow\infty}\frac{\weight_\calH(\Path[0:n])}{n}<0$~\cite{das2022stability}.
Thus, $\Path$ converges if and only if mean payoff of $\Path$ is less than zero. 
We say a $\rmdp$ is stable if an infinite path of it converges in expectation under any policy. A $\pphs$ is said to be stable when its semantics $\rmdp$ is stable.
\begin{definition}
    [Stability of PPHS]\label{def:stab-pphs} A $\rmdp$ $\Mw$ is said to be stable if under any policy $\policy$, an infinite path of it converges in expectation, i.e.,
    $\avg_\policy[\meanpay] < 0$.
    A $\pphs$ $\calH$ is said to be stable when its semantics $\rmdp$ $\semantics{\calH}$ is stable.
\end{definition}
For a $\rmdp$, an infinite path converges in expectation under all policies, if and only if, maximum expected mean payoff of an infinite path is less than zero. Thus, we have the following characterization for stability of a $\pphs$,
\begin{theorem}\label{thm:char-stab}
[Characterization of Stability] A $\pphs$ $\calH$ is stable iff, maximum expected mean payoff of an infinite path of its semantics $\rmdp$ $\semantics{\calH}$ is strictly negative, i.e.,
\[
\sup_{\policy\in\polset}\avg_\policy[\meanpay]<0.
\]
\end{theorem}
\begin{proof}
A $\pphs$ $\calH$ is stable when its semantics $\rmdp$ $\semantics{\calH}$ is stable. Now,
    $\semantics{\calH}$ is stable if and only if
    $\avg_\policy[\meanpay] < 0$ $\forall\policy\in\polset$, if and only if
    $\sup_{\policy\in\polset}\avg_\policy[\meanpay]<0$.
Hence, our claim is proved.\qed
\end{proof}
However, the algorithm discussed in Section~\ref{algo:lp-gain} cannot be applied to $\semantics{\calH}$ since it has infinite state and action space. 
Hence, we create an abstract $\rmdp$ $\reduce{\calH}$ from $\semantics{\calH}$ which is finite in size and, stability of $\reduce{\calH}$ implies stability of $\semantics{\calH}$.

\begin{remark}
Different notions of stability, such as, exponential stability, Lyapunov stability, Lagrange stability, asymptotic stability, have been explored for non-probabilistic hybrid systems. 
These notions have been extended for $\shs$~\cite{teel2014stability} by defining them on expected behavior of the system. 
For example, an $\shs$ is said to be Lyapunov stable around a point of equilibrium when the system remains within a close neighborhood of the equilibrium point \emph{in the expected case}, when it starts from a point close to the equilibrium point. 
We can say that, our notion of stability extends the notion of asymptotic stability in the non-stochastic setting, since we obtain the definition of asymptotic stability in non-stochastic setting by removing the `in expectation' part from Definition~\ref{def:stab-pphs}.
\end{remark}

\section{Abstraction of PPHS to WMDP}
We now describe the abstract $\rmdp$ $\reduce{\calH}$ derived from the semantics $\rmdp$ $\semantics{\calH}$ of a $\pphs$ $\calH$, and show that, stability of $\reduce{\calH}$ implies stability of $\semantics{\calH}$.
\begin{definition}
	[Abstract $\rmdp$] Let $\calH$ be a $\pphs$ and $\semantics{\calH}$ be its 
	semantics. We define the $\rmdp$ 
	$\reduce{\calH}=(\reduce{\states},\reduce{\acts},\reduce{\trans},\reduce{\weight})$ as follows,
	\begin{itemize}
		\item $\reduce{\states} = \loc\times\bigunion_{q\in\loc}\face(\inv(q))$
		\item $\reduce{\acts}\subseteq\reduce{\states}\times\reduce{\states}$ with $((q_1,f_1),(q_2,f_2))\in\acts$ iff $q_1=q_2$ and there is a continuous transition from $x_1\in f_1$ to $x_2\in f_2$ with respect to $q_1$.
		\item $\reduce{\trans}((q_2,f_2)|(q_1,f_1),a)
            =\trans_\calH((q_2,x_2)|(q_1,x_1),a')$ if $a=((q_1,f_1),(q_1,f_2))$, $x_1 \in
            f_1$, $x_2 \in f_2$ and $\trans_\calH((q_2,x_2)|(q_1,x_1),a')>0$. In all other cases, $\reduce{\trans}((q_2,f_2)|(q_1,f_1),a)=0$.
        \item $\reduce{\weight}((q_1,f_1),((q_1,f_1),(q_1,f_2)))= \max\{\weight_\calH((q_1,x_1),a')\}$, where maximum is taken over all $a'$ of the form $((q_1,x_1),(q_2,x_2))$ such that $x_1\in f_1$ and $x_2\in f_2$.
		\end{itemize}
\end{definition}

Note that, $\reduce{\calH}$ has finite state space and action space since for any $q\in\loc$, $\face(\inv(q))$ is finite. Also, $\max\{\weight_\calH((q_1,x_1),a',(q_2,x_2))\}$,
where maximum is taken over all $((q_1,x_1),a',(q_2,x_2))\in\edgeset_\calH$ such that $x_1\in f_1$ and $x_2\in f_2$, can be calculated by solving linear optimization problems when $\inv(q_1)$ is positive scaling invariant (see~\cite{prabhakar2013abstraction}). 

\begin{example}
For example, let us abstract the $\pphs$ described in Example~\ref{example:phs}. Assuming $(1,0)$ to be the initial state and $q_0$ the initial location, $(q_0,\{x\geq 0,y=0\})$ gives the initialization point for the abstract $\rmdp$. A unique action $((q_0,\{x\geq 0,y=0\}),(q_0,\{x=0,y\geq0\}))$ leads from the initialization point to the distribution $\tau_1$ with non-zero weight. $\tau_1$ has probability $p_1$ for state $(q_1,\{x=0,y\geq0\})$ and probability $1-p_1$ for state $(q_2,\{x=0,y\geq0\})$. Other states and transitions are defined similarly. The resulting $\rmdp$ is depicted by Figure~\ref{fig:mdp}, where $s_0$ marks the initialization point.\qed
\end{example}

We will now prove that $\calH$ is stable if $\reduce{\calH}$ is stable. 
This result will lead to a polynomial time algorithm that can verify stability of $\pphs$.

\begin{theorem}\label{thm:eps-abs-conc}
$\calH$ is stable if $\reduce{\calH}$ is stable.
\end{theorem}

\begin{proof}
Let $\policy$ be an arbitrary policy of $\semantics{\calH}$. Suppose, $(q_0,x_0)$ is the initialization point of $\semantics{\calH}$ and $x_0$ belongs to the facet $f_0$. Then, $(q_0,f_0)$ is the initialization point of $\reduce{\calH}$. We say that a state $(q,x)$ of $\semantics{\calH}$ belongs to a state $(q',f')$ of $\reduce{\calH}$, denoted as $(q,x)\in(q',f')$, if $q=q'$ and $x\in f'$. We define a policy $\hat{\policy}$ for $\reduce{\calH}$ using $\policy$ as follows:
\begin{itemize}
    \item Let $\hat{\Path}=(q_0,f_0),((q_0,f_0),(q_0,f_1)),(q_1,f_1),\dots,(q_n,f_n)$ be a finite path of $\reduce{\calH}$ and the action $((q_n,f_n),(q_n,f'))$ $\in\reduce{\acts}$.
    \item Then,
    \begin{align*}
        &\hat{\policy}(\hat{\Path}) = ((q_n,f_n),(q_n,f'))\\
        &\text{iff, }\exists\Path\in\fpath{\semantics{\calH}}, \Path_i\in\hat{\Path}_i\quad \forall i\in[n],\text{ and }x'\in f',
    \end{align*}
    where $\policy(\Path)=((q_n,x_n),(q_n,x'))$.
\end{itemize}
This implies by construction of $\reduce{\calH}$ that, probability of a finite path $\hat{\Path}$ under policy $\hat{\policy}$ is simply the probability (under $\policy$) of set of all finite paths $\Path$ of $\semantics{\calH}$ such that $\size{\hat{\Path}}=\size{\Path}$ and $\Path_i\in \hat{\Path}_i$ for all $i$. 
Also note that, $\reduce{\weight}(\hat{\Path})\geq\weight_\calH(\Path)$ for any such $\Path\in\fpath{\semantics{\calH}}$. 
In fact, for $\hat{\Path}\in\fpath{\reduce{\calH}}$, let 
\[
\Pi(\hat{\Path}) = \{\Path\in\fpath{\semantics{\calH}}:\Path_i\in\hat{\Path}_i\text{ }\forall i\};
\]
then, for each $\Path\in\Pi(\hat{\Path})$, 
\begin{align*}
&\frac{\weight_\calH(\Path)}{\size{\Path}} \leq
\frac{\reduce{\weight}(\hat{\Path})}{\size{\hat{\Path}}}\text{ and,}\\
& \trans_{\hat{\policy}}(\hat{\Path}) = \trans_\policy\left(\Pi(\hat{\Path})\right) \quad\text{for all }\hat{\Path}\in\fpath{\reduce{\calH}}.
\end{align*}
Since for any $\Path\in\fpath{\semantics{\calH}}$, there exists $\hat{\Path}\in\fpath{\reduce{\calH}}$ such that $\size{\Path}=\size{\hat{\Path}}$ and $\Path_i\in\hat{\Path}_i$ for all $i$, 
\[
\bigunion_{\hat{\Path}\in\fpath{\reduce{\calH}}}\Pi(\hat{\Path})=\fpath{\semantics{\calH}}.
\]
Thus, expected mean payoff of an infinite path of $\reduce{\calH}$ under $\hat{\policy}$,
\begin{align*}
\avg_{\hat{\policy}}[\meanpay] =&\limsup_{n\rightarrow\infty}\sum_{\hat{\Path}\in\fpath{\reduce{\calH}}}\left(\frac{\reduce{\weight}(\hat{\Path})}{\size{\hat{\Path}}}\right)\trans_n^{\hat{\policy}}(\hat{\Path})\\
= &\limsup_{n\rightarrow\infty}\sum_{\hat{\Path}\in\fpath{\reduce{\calH}}}\left(\frac{\reduce{\weight}(\hat{\Path})}{\size{\hat{\Path}}}\right)\trans_n^{\policy}(\Pi(\hat{\Path}))\\
\geq &
\limsup_{n\rightarrow\infty}\sum_{\Path\in\fpath{\semantics{\calH}}}\left(\frac{\weight_\calH(\Path)}{\size{\Path}}\right)\trans_n^{\policy}(\Path)=\avg_\policy[\meanpay],
\end{align*}
that is, expected mean payoff of an infinite path of $\semantics{\calH}$ under $\policy$.
Since for any arbitrary policy $\policy$ of $\semantics{\calH}$, we can define a policy $\hat{\policy}$ of $\reduce{\calH}$ such that $\avg_{\hat{\policy}}[\meanpay]\geq\avg_\policy[\meanpay]$, we can say that,
\[
\sup_{\hat{\policy}\in\polset_{\reduce{\calH}}} \avg[\meanpay] \geq \sup_{\policy\in\polset_{\semantics{\calH}}} \avg[\meanpay].
\]
Thus, stability of $\reduce{\calH}$, i.e.,
    $\avg_{\hat{\policy}}[\meanpay] < 0$ for all $\hat{\policy}\in\polset_{\reduce{\calH}}$, implies 
    $\sup_{\hat{\policy}\in\polset_{\reduce{\calH}}} \avg_{\hat{\policy}}[\meanpay] < 0$,
    which further implies $\sup_{\policy\in\polset_{\semantics{\calH}}} \avg_\policy[\meanpay] < 0$,
which, by Theorem~\ref{thm:char-stab} implies that $\calH$ is stable too. Hence, our claim is proved.\qed
\end{proof}
Using Theorem \ref{thm:eps-abs-conc}, we can easily device a polynomial time algorithm that can verify stability of a $\pphs$ $\calH$. We simply construct $\reduce{\calH}$ from $\semantics{\calH}$, which takes polynomial time according to~\cite{prabhakar2013abstraction}, and apply the polynomial time algorithm discussed in Section~\ref{algo:lp-gain} on $\reduce{\calH}$ to find the maximum expected mean payoff of an infinite path of $\reduce{\calH}$. If the maximum expected mean payoff is less than zero, then we deduce that $\calH$ is stable. 
Note that, our algorithm tests a sufficient condition for stability, i.e., we cannot say that the $\pphs$ is unstable if the algorithm does not guarantee stability.

\begin{remark}
We would like to remark on the generality of this abstraction procedure. 
Note that, Theorem~\ref{thm:eps-abs-conc} holds even if we define stability using a different payoff function, like total effective payoff~\cite{boros2018markov} or, prefix-independent and submixing payoff~\cite{gimbert2007pure}. 
Not only that, it works for other notions of stability as well, such as almost sure stability. 
However, the algorithm for verifying stability of the abstract $\rmdp$ (Section~\ref{algo:lp-gain}) is not so general and needs to be modified for different notions of stability.
\end{remark}
\section{Experimental Evaluation}\label{sec:exp}

\begin{table}[h]
\centering
\resizebox{0.9\textwidth}{!}{%
\begin{tabular}{|c|c|c|c|cc|c|}
\hline
\multirow{2}{*}{$n$} &
  \multirow{2}{*}{Expt No.} &
  \multirow{2}{*}{Locs} &
  \multirow{2}{*}{$\size{\edgeset_{\Mw}}$} &
  \multicolumn{2}{c|}{Time (sec)} &
  \multirow{2}{*}{Stability} \\ \cline{5-6}
                     &   &      &        & \multicolumn{1}{c|}{$\reduce{T}$} & $T^{\text{stab}}$   &         \\ \hline
\multirow{4}{*}{$2$} & 1 & $4$  & $16$   & \multicolumn{1}{c|}{$0.113$}      & $1.8\times10^{-4}$ & Unknown \\ \cline{2-7} 
                     & 2 & $8$  & $64$   & \multicolumn{1}{c|}{$0.242$}      & $0.143$            & Yes     \\ \cline{2-7} 
                     & 3 & $12$ & $144$  & \multicolumn{1}{c|}{$0.367$}      & $0.149$            & Yes     \\ \cline{2-7} 
                     & 4 & $16$ & $256$  & \multicolumn{1}{c|}{$0.517$}      & $0.158$            & Yes     \\ \hline
\multirow{4}{*}{$3$} & 1 & $8$  & $48$   & \multicolumn{1}{c|}{$3.042$}      & $7.3\times10^{-4}$ & Unknown \\ \cline{2-7} 
                     & 2 & $16$ & $192$  & \multicolumn{1}{c|}{$6.059$}      & $0.213$            & Yes     \\ \cline{2-7} 
                     & 3 & $24$ & $432$  & \multicolumn{1}{c|}{$9.203$}      & $0.311$            & Yes     \\ \cline{2-7} 
                     & 4 & $32$ & $768$  & \multicolumn{1}{c|}{$12.213$}     & $0.438$            & Yes     \\ \hline
\multirow{4}{*}{$4$} & 1 & $4$  & $24$   & \multicolumn{1}{c|}{$2.676$}      & $5\times10^{-4}$   & Unknown \\ \cline{2-7} 
                     & 2 & $8$  & $96$   & \multicolumn{1}{c|}{$5.784$}      & $0.163$            & Yes     \\ \cline{2-7} 
                     & 3 & $12$ & $216$  & \multicolumn{1}{c|}{$9.018$}      & $0.197$            & Yes     \\ \cline{2-7} 
                     & 4 & $16$ & $384$  & \multicolumn{1}{c|}{$11.678$}     & $0.244$            & Yes     \\ \hline
\multirow{4}{*}{$5$} & 1 & $10$ & $90$   & \multicolumn{1}{c|}{$11.986$}     & $1.4\times10^{-3}$ & Unknown \\ \cline{2-7} 
                     & 2 & $20$ & $360$  & \multicolumn{1}{c|}{$25.098$}     & $0.275$            & Yes     \\ \cline{2-7} 
                     & 3 & $30$ & $810$  & \multicolumn{1}{c|}{$37.356$}     & $0.438$            & Yes     \\ \cline{2-7} 
                     & 4 & $40$ & $1440$ & \multicolumn{1}{c|}{$49.743$}     & $0.67$             & Yes     \\ \hline
\multirow{4}{*}{$6$} & 1 & $20$ & $240$  & \multicolumn{1}{c|}{$38.377$}     & $4\times10^{-3}$   & Unknown \\ \cline{2-7} 
                     & 2 & $40$ & $960$  & \multicolumn{1}{c|}{$78.049$}     & $0.631$            & Yes     \\ \cline{2-7} 
                     & 3 & $60$ & $2160$ & \multicolumn{1}{c|}{$120.906$}    & $1.241$            & Yes     \\ \cline{2-7} 
                     & 4 & $80$ & $3840$ & \multicolumn{1}{c|}{$160.049$}    & $2.08$             & Yes     \\ \hline
\end{tabular}%
}
\setlength\abovecaptionskip{10pt}
\setlength\belowcaptionskip{-5pt}
\caption{Verification of stability of PPHS}
\label{table:stab-verify}
\end{table}

In this section, we provide a brief overview of our implementation of the abstraction based stability verification algorithm of $\pphs$ and test it on a set of $\pphs$ benchmarks. 
Recall that, our technique consists of abstraction of a $\pphs$ to a finite $\rmdp$ and verification of stability of the abstract $\rmdp$ using the algorithm developed in Section~\ref{algo:lp-gain}.
We have implemented the abstraction procedure and the algorithm using Python. The abstract $\rmdp$ is stored as annotated graph using the \emph{networkx}~\cite{SciPyProceedings_11} package. In order to calculate each edge weight, $4n^2$ linear programming problems has to be solved, where $n$ is the number of dimensions~\cite{prabhakar2013abstraction}. This has been done using the \emph{pplpy} package. To find all MECs of the $\rmdp$ through strongly connected component decomposition~\cite{chatterjee2013faster}, \emph{networkx} functions are used. For linear programming, the software \emph{Gurobi} and its python handler \emph{gurobipy} are used. All experiments have been performed on macOS Big Sur with Quad-Core (Intel Core i7) $2.8\text{GHz}\times1$ Processor and 16GB RAM.

We have tested the effect of increasing the number of dimensions and the number of locations of the $\pphs$ on the time requirement of our algorithm. 
We have varied the number of dimensions $n$ from $2$ to $6$ and created four sample $\pphs$ for each dimension. 
For dimension $2$ and $3$, four quadrants of $\real^2$ and eight octants of $\real^3$ are chosen respectively as the invariant regions. 
For higher dimensions ($n=4$, $5$ and $6$), $3$-dimensional hyperplanes of $\real^n$ are chosen as the invariant regions. 
For experiment $j$ ($j\in[4]$) of dimension $n$, $j$ locations are created for each invariant region. 
For example, the $\pphs$ corresponding to experiment $2$ of dimension $3$ has $2$ locations per octant, i.e., $2\times 8=16$ locations in total. 
For experiment $1$ in all dimensions, the flow polyhedron is set for each location such that at least one edge of the abstract $\rmdp$ has infinite weight. For experiments $3$ and $4$, flows are set such that at least one edge of the $\rmdp$ has positive weight. For experiment $2$ however, all flows are set such that no edge of the abstract $\rmdp$ can have positive weight.

We present our findings in Table~\ref{table:stab-verify}. Here $n$ is number of dimensions, $\text{Locs}$ is the total number of locations of the $\pphs$ corresponding to the experiment, $\size{\edgeset_{\Mw}}$ is the number of edges of the abstract $\rmdp$, $\reduce{T}$ denotes the time taken to generate the abstract $\rmdp$ and $T^{\text{stab}}$ denotes the time taken to verify stability of the abstract $\rmdp$. For each experiment, the average time over $50$ runs is reported. Finally, the Stability column provides the information on whether the $\pphs$ is stable or not. 
Note that, a $\pphs$ is stable if the abstract $\rmdp$ is stable, not necessarily the other way round. 
Thus, we cannot say a $\pphs$ is unstable if the stability checking algorithm (Section~\ref{algo:lp-gain}) designates the abstract $\rmdp$ to be unstable. 
Hence, in such cases, we have reported ``Unknown" in the Stability column. 
In other cases, we have reported ``Yes" in the Stability column, which means that the corresponding $\pphs$ is found to be stable. 
We observe that the abstraction time dominates over the verification time in all cases. 
In fact, the abstraction time increases rapidly with the number of dimensions $n$. This is expected since calculation of each edge weight of the abstract $\rmdp$ requires solving $4n^2$ linear programs~\cite{prabhakar2013abstraction}. 
Increase in the number of locations results in increased abstraction time as well since number of edges of the abstract $\rmdp$ increases. 
The verification time however, does not always increase with the number of dimensions. This is because, the verification time depends on the number of edges of the abstract $\rmdp$ and not on the number of dimensions.
For experiment $1$ however, verification time is extremely small as the algorithm has found an infinite 
weighted edge and deduces the abstract $\rmdp$ to be unstable without going through 
MEC-decomposition or LP solving.

For our second set of experiments, we test our algorithm on Linear Switched 
Systems~\cite{prajna2003analysis}. 
For Linear Switched System with $\real^n$ as the continuous state space, rate of change of the 
continuous state $v$ is a linear function of the current state, and this linear function changes 
arbitrarily. 
More precisely, the system evolving with dynamics $\dot{v}=A_1v$ ($A_1\in\real^{n\times n}$) can 
change its dynamics to $\dot{v}=A_2v$ ($\real^{n\times n}\ni A_2\neq A_1$) arbitrarily. 
We take Example 2 from~\cite{prajna2003analysis} where the system evolves in $\real^2$ with 
dynamics $\dot{v}=Av$, and $A$ switches between
\begin{equation*}
	A_1 = 
	\begin{bmatrix}
		-5 &\quad -4\\
		-1 &\quad -2
	\end{bmatrix}
	\text{ and }
	A_2 = 
	\begin{bmatrix}
		-2 &\quad -4\\
		20 &\quad -2
	\end{bmatrix}.
\end{equation*}
Instead of arbitrary switching, we assume probabilistic switching of dynamics, i.e., at any point of time 
the system either retains its dynamics or changes it (if possible) with equal probability. 
To analyze stability of this system, we apply hybridization technique discussed 
in~\cite{prabhakar2016hybridization}.
Basically, we partition the continuous state space into positively scaled regions and associate two 
locations to each partition such that, for the $i^{th}$ location, dynamics is given by $\dot{v}\in F_i$ 
($F_i$ is the polyhedron formed by points $A_iv$, where $v$ belongs to the partition). 
Change of location is allowed at the boundary of the corresponding partition only, at which point, one 
of the locations from the adjacent partition is chosen with equal probability. 
Clearly, this process generates a $\pphs$. 
We analyze stability of the $\pphs$ using our algorithm and generate $\pphs$ with finer partitions until 
stability is ensured. 
Note that, stability of the $\pphs$ implies stability of the Linear Switched 
System~\cite{prabhakar2016hybridization}.
In fact, we used the same partitions as in test cases SS4\_1, SS8\_1 and SS16\_1 
of~\cite{prabhakar2016hybridization}.
Stability is ensured for both $\pphs$ corresponding to  SS8\_1 and SS16\_1 and not for the $\pphs$ 
corresponding to SS4\_1, which matches with the observations of~\cite{prabhakar2016hybridization}.

\section{Conclusion}

In this paper, we have presented an algorithm for stability analysis of an important subclass of Stochastic Hybrid Systems, which we call the Polyhedral Probabilistic Hybrid System ($\pphs$). 
Our algorithm is based on abstraction based techniques, where we first abstract the $\pphs$ to a finite $\rmdp$ and then test the abstract $\rmdp$ for stability. 
Verification of stability of the abstract $\rmdp$ is extremely efficient, since it can be done using a polynomial-time algorithm. 
However, the abstraction time increases rapidly with the number of dimensions and always dominates over the verification time.
Hence, the entire process suffers from the curse of dimensionality. 
For our future work, we would like to develop compositional methods for analyzing stochastic stability to circumvent this problem.
Two other directions of future research are, exploring probabilistic stability analysis for almost sure notions and for more complex dynamics, including those with stochasticity in the continous dynamics and, analyzing stability in a chosen set of dimensions instead of all the dimensions, which is hard to achieve in reality.

%
%
%
\bibliographystyle{splncs04}
\bibliography{spandan}

\end{document}